\documentclass[twoside,11pt]{article}

\usepackage{fullpage}
\usepackage{graphics}
\usepackage{amssymb}
\usepackage{latexsym}
\usepackage{amsthm}
\usepackage{amsmath}

\newtheorem{theorem}{Theorem}
\newtheorem{lemma}[theorem]{Lemma}
\newtheorem{corollary}[theorem]{Corollary}
\newtheorem{definition}[theorem]{Definition}
\newtheorem{example}{Eaxmple}

\newcommand{\Sec}[1]{Section~\ref{#1}}
\newcommand{\Tab}[1]{Table~\ref{#1}}

\newcommand{\Eqn}[1]{Eq.~(\ref{#1})}
\newcommand{\Eqs}[2]{Eqs.~(\ref{#1}-\ref{#2})}

\newcommand{\Thm}[1]{Theorem~\ref{#1}}
\newcommand{\Cor}[1]{Corollary~\ref{#1}}
\newcommand{\App}[1]{Appendix~\ref{#1}}

\renewcommand{\>}{{\rightarrow}}

\newcommand{\half}{{\textstyle{\frac{1}{2}}}}
\renewcommand{\hat}{\widehat}
\renewcommand{\tilde}{\widetilde}
\newcommand{\R}{{\mathbb R}}

\renewcommand{\P}{{\mathbf P}}
\newcommand{\E}{{\mathbf E}}

\newcommand{\1}{{\mathbf 1}}

\newcommand{\X}{{\mathcal X}}
\newcommand{\Y}{{\mathcal Y}}

\newcommand{\sign}{\textup{\textrm{sign}}}
\newcommand{\er}{\textup{\textrm{er}}}
\newcommand{\reg}{\textup{\textrm{regret}}}
\newcommand{\NA}{\textup{\textrm{NA}}}
\newcommand{\zo}{\textup{\textrm{\scriptsize 0-1}}}
\newcommand{\sq}{\textup{\textrm{\scriptsize sq}}}

\newcommand{\spher}{\textup{\textrm{\scriptsize spher}}}
\newcommand{\rank}{\textup{\textrm{\scriptsize rank}}}
\newcommand{\bal}{\textup{\textrm{\scriptsize bal}}}

\newcommand{\diff}{\textup{\textrm{\scriptsize diff}}}  
\newcommand{\can}{\textup{\textrm{\scriptsize can}}}

\begin{document}

%---------- Title -------------

\title{Surrogate Regret Bounds for Bipartite Ranking\\ via Strongly Proper Losses}

\author{Shivani Agarwal \\
       Department of Computer Science and Automation \\
       Indian Institute of Science \\
       Bangalore 560012, India \\
       \texttt{shivani@csa.iisc.ernet.in}
}

\maketitle

%----------- Abstract -----------

\begin{abstract}
The problem of bipartite ranking, where instances are labeled positive or negative and the goal is to learn a scoring function that minimizes the probability of mis-ranking a pair of positive and negative instances (or equivalently, that maximizes the area under the ROC curve), has been widely studied in recent years. A dominant theoretical and algorithmic framework for the problem has been to reduce bipartite ranking to pairwise classification; in particular, it is well known that the bipartite ranking regret can be formulated as a pairwise classification regret, which in turn can be upper bounded using usual regret bounds for classification problems. Recently, Kotlowski et al.\ (2011) showed regret bounds for bipartite ranking in terms of the regret associated with balanced versions of the standard (non-pairwise) logistic and exponential losses. In this paper, we show that such (non-pairwise) surrogate regret bounds for bipartite ranking can be obtained in terms of a broad class of proper (composite) losses that we term as \emph{strongly proper}. Our proof technique is much simpler than that of Kotlowski et al.\ (2011), and relies on properties of proper (composite) losses as elucidated recently by Reid and Williamson (2010, 2011) and others. Our result yields explicit surrogate bounds (with no hidden balancing terms) in terms of a variety of strongly proper losses, including for example logistic, exponential, squared and squared hinge losses as special cases. We also obtain tighter surrogate bounds under certain low-noise conditions via a recent result of Cl\'{e}men\c{c}on and Robbiano (2011).
\end{abstract}

%========== SECTION 1 ===========
\section{Introduction}
\label{sec:intro}

Ranking problems arise in a variety of applications ranging from information retrieval to recommendation systems and from computational biology to drug discovery, and have been widely studied in machine learning and statistics in the last several years. Recently, there has been much interest in understanding statistical consistency and regret behavior of algorithms for a variety of ranking problems, including various forms of label/subset ranking as well as instance ranking problems \cite{ClemenconVa07,Clemencon+08,CossockZh08,Balcan+08,AilonMo08,Xia+08,Duchi+10,Ravikumar+10,Buffoni+11,ClemenconRo11,Kotlowski+11,UematsuLe11}.

In this paper, we study regret bounds for the bipartite instance ranking problem, where instances are labeled positive or negative and the goal is to learn a scoring function that minimizes the probability of mis-ranking a pair of positive and negative instances, or equivalently, that maximizes the area under the ROC curve \cite{Freund+03,Agarwal+05}. A popular algorithmic and theoretical approach to bipartite ranking has been to treat the problem as analogous to pairwise classification \cite{Herbrich+00,Joachims02,Freund+03,Rakotomamonjy04,Burges+05,Clemencon+08}. Indeed, this approach enjoys theoretical support since the bipartite ranking regret can be formulated as a pairwise classification regret, and therefore any algorithm minimizing the latter over a suitable class of functions will also minimize the ranking regret (this follows formally from results of \cite{Clemencon+08}; see \Sec{subsec:reduction-pairwise} for a summary). Nevertheless, it has often been observed that algorithms such as AdaBoost, logistic regression, and in some cases even SVMs, which minimize the exponential, logistic, and hinge losses respectively in the standard (non-pairwise) setting, also yield good bipartite ranking performance \cite{CortesMo04,Rakotomamonjy04,RudinSc09}. For losses such as the exponential or logistic losses, this is not surprising since algorithms minimizing these losses (but not the hinge loss) are known to effectively estimate conditional class probabilities \cite{Zhang04}; since the class probability function provides the optimal ranking \cite{Clemencon+08}, it is intuitively clear (and follows formally from results in \cite{Clemencon+08,ClemenconRo11}) that any algorithm providing a good approximation to the class probability function should also produce a good ranking. However, there has been very little work so far on quantifying the ranking regret of a scoring function in terms of the regret associated with such surrogate losses.

Recently, \cite{Kotlowski+11} showed that the bipartite ranking regret of a scoring function can be upper bounded in terms of the regret associated with \emph{balanced} versions of the standard (non-pairwise) exponential and logistic losses. However their proof technique builds on analyses involving the reduction of bipartite ranking to pairwise classification, and involves analyses specific to the exponential and logistic losses (see \Sec{subsec:kotlowski+11}). More fundamentally, the balanced losses in their result depend on the underlying distribution and cannot be optimized directly by an algorithm; while it is possible to do so approximately, one then loses the quantitative nature of the bounds.

In this work we obtain quantitative regret bounds for bipartite ranking in terms of a broad class of proper (composite) loss functions that we term \emph{strongly proper}. Our proof technique is considerably simpler than that of \cite{Kotlowski+11}, and relies on properties of proper (composite) losses as elucidated recently for example in \cite{ReidWi10,ReidWi11,GneitingRa07,Buja+05}. Our result yields explicit surrogate bounds (with no hidden balancing terms) in terms of a variety of strongly proper (composite) losses, including for example logistic, exponential, squared and squared hinge losses as special cases. We also obtain tighter surrogate bounds under certain low-noise conditions via a recent result of \cite{ClemenconRo11}.

The paper is organized as follows. In \Sec{sec:setup} we formally set up the bipartite instance ranking problem and definitions related to loss functions and regret, and provide background on proper (composite) losses. \Sec{sec:related} summarizes related work that provides the background for our study, namely the reduction of bipartite ranking to pairwise binary classification and the result of \cite{Kotlowski+11}. In \Sec{sec:strongly-proper} we define and characterize strongly proper losses. \Sec{sec:bounds} contains our main result, namely a bound on the bipartite ranking regret in terms of the regret associated with any strongly proper loss, together with several examples. \Sec{sec:tighter-bounds} gives a tighter bound under certain low-noise conditions via a recent result of \cite{ClemenconRo11}. We conclude with a brief discussion and some open questions in \Sec{sec:concl}.

%========== SECTION 2 ===========
\section{Formal Setup, Preliminaries, and Background}
\label{sec:setup}

This section provides background on the bipartite ranking problem, binary loss functions and regret, and proper (composite) losses.

%---------- Section 2.1 ----------
\subsection{Bipartite Ranking}
\label{subsec:bipartite}

As in binary classification, in bipartite ranking there is an instance space $\X$ and binary labels $\Y=\{\pm1\}$, with an unknown distribution $D$ on $\X\times\{\pm1\}$. For $(X,Y)\sim D$ and $x\in\X$, we denote $\eta(x) = \P(Y=1\mid X=x)$ and $p = \P(Y=1)$. Given i.i.d.\ examples $(X_1,Y_1),\ldots,(X_n,Y_n) \sim D$, the goal is to learn a scoring function $f:\X\>\R^*$ (where $\R^* = [-\infty,\infty]$) that assigns higher scores to positive instances than to negative ones.\footnote{Most algorithms learn real-valued functions; we also allow values $-\infty$ and $\infty$ for technical reasons.}  Specifically, the goal is to learn a scoring function $f$ with low \emph{ranking error} (or \emph{ranking risk}), defined as\footnote{We assume measurability  conditions where necessary.}
\begin{equation}
\er_D^\rank[f] 
    ~ = ~ 
    \E\Big[ \1\big( (Y-Y')(f(X)-f(X')) < 0 \big) + \half\, \1\big( f(X)=f(X') \big) ~\big|~ Y\neq Y' \Big]
    \,,
\end{equation}
where $(X,Y), (X',Y')$ are assumed to be drawn i.i.d.\ from $D$, and $\1(\cdot)$ is 1 if its argument is true and 0 otherwise; thus the ranking error of $f$ is simply the probability that a randomly drawn positive instance receives a lower score under $f$ than a randomly drawn negative instance, with ties broken uniformly at random. 
The \emph{optimal ranking error} (or \emph{Bayes ranking error} or \emph{Bayes ranking risk}) can be seen to be
\begin{eqnarray}
\er_D^{\rank,*} 
    & = &
    \inf_{f:\X\>\R^*} \er_D^\rank[f] 
\\
    & = &
    \frac{1}{2p(1-p)} \E_{X,X'}\Big[ \min\Big( \eta(X)(1-\eta(X')), \, \eta(X')(1-\eta(X)) \Big) \Big]
    \,.
\end{eqnarray}
The \emph{ranking regret} of a scoring function $f:\X\>\R^*$ is then simply
\begin{equation}
\reg_D^\rank[f]
    ~ = ~
    \er_D^\rank[f] - \er_D^{\rank,*}
    \,.
\end{equation}
We will be interested in upper bounding the ranking regret of a scoring function $f$ in terms of its regret with respect to certain other (binary) loss functions. In particular, the loss functions we consider will belong to the class of \emph{proper} (composite) loss functions. Below we briefly review some standard notions related to loss functions and regret, and then discuss some properties of proper (composite) losses.

%---------- Section 2.2 ----------
\subsection{Loss Functions, Regret, and Conditional Risks and Regret}
\label{subsec:loss}

Assume again a probability distribution $D$ on $\X\times\{\pm1\}$ as above. Given a prediction space $\hat{\Y}\subseteq\R^*$, a binary \emph{loss function} $\ell:\{\pm1\}\times\hat{\Y}\>\R_+^*$ (where $\R_+^* = [0,\infty]$) assigns a penalty $\ell(y,\hat{y})$ for predicting $\hat{y}\in\hat{\Y}$ when the true label is $y\in\{\pm1\}$.\footnote{Most loss functions take values in $\R_+$, but some loss functions (such as the logistic loss, described later) can assign a loss of $\infty$ to certain label-prediction pairs.}
For any such loss $\ell$, the \emph{$\ell$-error} (or \emph{$\ell$-risk}) of a function $f:\X\>\hat{\Y}$ is defined as
\begin{equation}
\er_D^\ell[f] = 
	\E_{(X,Y)\sim D} [\ell(Y,f(X))]
	\,,
\end{equation}
and the \emph{optimal $\ell$-error} (or \emph{optimal $\ell$-risk} or \emph{Bayes $\ell$-risk}) is defined as
\begin{equation}
\er_D^{\ell,*} = \inf_{f:\X\>\hat{\Y}} \er_D^\ell[f]
	\,.
\end{equation} 
The \emph{$\ell$-regret} of a function $f:\X\>\hat{\Y}$ is the difference of its $\ell$-error from the optimal $\ell$-error:
\begin{equation}
\reg_D^\ell[f] = 
	\er_D^\ell[f] - \er_D^{\ell,*}
	\,.
\end{equation}
The \emph{conditional $\ell$-risk} $L_\ell:[0,1]\times\hat{\Y}\>\R_+^*$ is defined as\footnote{Note that we overload notation by using $\eta$ here to refer to a number in $[0,1]$; the usage should be clear from context.}
\begin{equation}
L_\ell(\eta,\hat{y}) = 
	\E_{Y\sim\eta}[\ell(Y,\hat{y})] = 
	\eta\, \ell(1,\hat{y}) + (1-\eta)\, \ell(-1,\hat{y})
	\,,
\end{equation}
where $Y\sim\eta$ denotes a $\{\pm1\}$-valued random variable taking value $+1$ with probability $\eta$.
The \emph{conditional Bayes $\ell$-risk} $H_\ell:[0,1]\>\R_+^*$ is defined as 
\begin{equation}
H_\ell(\eta) = 
	\inf_{\hat{y}\in\hat{\Y}} L_\ell(\eta,\hat{y})
	\,.
\end{equation}
The \emph{conditional $\ell$-regret} $R_\ell:[0,1]\times\hat{\Y}\>\R_+^*$ is then simply
\begin{equation}
R_\ell(\eta,\hat{y}) = 
	L_\ell(\eta,\hat{y}) - H_\ell(\eta)
	\,.
\end{equation}
Clearly, we have for $f:\X\>\hat{\Y}$,
\begin{equation}
\er_D^\ell[f] = \E_{X}[L_\ell(\eta(X),f(X))]
	\,,
\end{equation}
and
\begin{equation}
\er_D^{\ell,*} = \E_{X}[H_\ell(\eta(X))]
	\,.
\end{equation}
We note the following:
\begin{lemma}
For any $\hat{\Y}\subseteq\R^*$ and binary loss $\ell:\{\pm1\}\times\hat{\Y}\>\R_+^*$, the conditional Bayes $\ell$-risk $H_\ell$ is a concave function on $[0,1]$.
\end{lemma}
The proof follows simply by observing that $H_\ell$ is defined as the pointwise infimum of a family of linear (and therefore concave) functions, and therefore is itself concave.

%---------- Section 2.3 ----------
\subsection{Proper and Proper Composite Losses}
\label{subsec:proper}

In this section we review some background material related to proper and proper composite losses, as studied recently in \cite{ReidWi10,ReidWi11,GneitingRa07,Buja+05}. While the material is meant to be mostly a review, some of the exposition is simplified compared to previous presentations, and we include a new, simple proof of an important fact (\Thm{thm:strictly-proper}).

%---------- Proper Losses ------------
\vspace{10pt}
\noindent
\textbf{Proper Losses.}
We start by considering binary class probability estimation (CPE) loss functions that operate on the prediction space $\hat{\Y}=[0,1]$. A binary CPE loss function $c:\{\pm1\}\times[0,1]\>\R_+^*$ is said to be \emph{proper} if for all $\eta\in[0,1]$,
\begin{equation}
\eta \in
    {\underset{\hat{\eta}\in[0,1]}{\arg\min}} \,
    L_c(\eta,\hat{\eta})
    \,,
\end{equation}
and \emph{strictly proper} if the minimizer is unique for all $\eta\in[0,1]$.
Equivalently, $c$ is proper if for all $\eta\in[0,1]$, $H_c(\eta) = L_c(\eta,\eta)$, and strictly proper if $H_c(\eta) < L_c(\eta,\hat{\eta})$ for all $\hat{\eta}\neq\eta$. We have the following basic result:
\begin{lemma}[\cite{GneitingRa07,Schervish89}]
\label{lem:increasing}
Let $c:\{\pm1\}\times[0,1]\>\R_+^*$ be a binary CPE loss. If $c$ is proper, then $c(1,\cdot)$ is a decreasing function on $[0,1]$ and $c(-1,\cdot)$ is an increasing function. If $c$ is strictly proper, then $c(1,\cdot)$ is strictly decreasing on $[0,1]$ and $c(-1,\cdot)$ is strictly increasing.
\end{lemma}
We will find it useful to consider \emph{regular} proper losses. 
As in \cite{GneitingRa07}, we say a binary CPE loss $c:\{\pm1\}\times[0,1]\>\R_+^*$ is \emph{regular} if $c(1,\hat{\eta}) \in\R_+ ~\forall \hat{\eta}\in(0,1]$ and $c(-1,\hat{\eta}) \in\R_+ ~\forall \hat{\eta}\in[0,1)$, i.e.\ if $c(y,\hat{\eta})$ is finite for all $y,\hat{\eta}$ except possibly for $c(1,0)$ and $c(-1,1)$, which are allowed to be infinite.
The following characterization of regular proper losses is well known (see also \cite{GneitingRa07}):
\begin{theorem}[\cite{Savage71}]
\label{thm:savage71}
A regular binary CPE loss $c:\{\pm1\}\times[0,1]\>\R_+^*$ is proper if and only if for all $\eta,\hat{\eta}\in[0,1]$ there exists a superderivative $H_c'(\hat{\eta})$ of $H_c$ at $\hat{\eta}$ such that\footnote{Here $u\in\R$ is a superderivative of $H_c$ at $\hat{\eta}$ if for all $\eta\in[0,1]$, $H_c(\hat{\eta}) - H_c(\eta) \geq u(\hat{\eta} - \eta)$.} 
\[
L_c(\eta,\hat{\eta}) 
	~ = ~ 
	H_c(\hat{\eta}) + (\eta - \hat{\eta}) \cdot H'_c(\hat{\eta})
	\,.
\]
\end{theorem}

The following is a characterization of strict properness of a proper loss $c$ in terms of its conditional Bayes risk $H_c$:
\begin{theorem}
\label{thm:strictly-proper}
A proper loss $c:\{\pm1\}\times[0,1]\>\R_+^*$ is strictly proper if and only if $H_c$ is strictly concave. 
\end{theorem}
This result can be proved in several ways. A proof in \cite{GneitingRa07} is attributed to an argument in \cite{HendricksonBu71}. If $H_c$ is twice differentiable, an alternative proof follows from a result in \cite{Buja+05,Schervish89}, which shows that a proper loss $c$ is strictly proper if and only if its `weight function' $w_c = -H''_c$ satisfies $w_c(\eta) > 0$ for all except at most  countably many points $\eta\in[0,1]$; by a very recent result of \cite{Stein11}, this condition is equivalent to strict convexity of the function $-H_c$, or equivalently, strict concavity of $H_c$. Here we give a third, self-contained proof of the above result that is derived from first principles, and that will be helpful when we study strongly proper losses in \Sec{sec:strongly-proper}.

\vspace{8pt}
\begin{proof}[Proof of \Thm{thm:strictly-proper}]
Let $c:\{\pm1\}\times[0,1]\>\R_+^*$ be a proper loss. For the `if' direction, assume $H_c$ is strictly concave. Let $\eta,\hat{\eta}\in[0,1]$ such that $\hat{\eta}\neq\eta$. Then we have
\begin{eqnarray*}
L_c(\eta,\hat{\eta}) - H_c(\eta)
	& = &
	L_c(\eta,\hat{\eta}) + H_c(\hat{\eta}) - H_c(\hat{\eta}) - H_c(\eta) 
\\
	& = &
	L_c(\eta,\hat{\eta}) + H_c(\hat{\eta}) - 2 \Big(\half H_c(\eta) + \half H_c(\hat{\eta}) \Big)
\\
	& > &
	L_c(\eta,\hat{\eta}) + H_c(\hat{\eta}) - 2 H_c\Big( \frac{\eta + \hat{\eta}}{2} \Big)
\\
	& = &
	2 \left(  \Big( \frac{\eta + \hat{\eta}}{2} \Big) c(1,\hat{\eta}) + \Big( 1 - \frac{\eta + \hat{\eta}}{2} \Big) c(-1,\hat{\eta}) \right)
	- 2 H_c\Big( \frac{\eta + \hat{\eta}}{2} \Big)
\\
	& = &
	2 \left( L_c\Big(  \frac{\eta + \hat{\eta}}{2}, \hat{\eta} \Big) - H_c\Big( \frac{\eta + \hat{\eta}}{2} \Big) \right)
\\	
	& \geq & 
	0
	\,.
\end{eqnarray*}
Thus $c$ is strictly proper. 

Conversely, to prove the `only if' direction, assume $c$ is strictly proper. Let $\eta_1,\eta_2\in[0,1]$ such that $\eta_1\neq \eta_2$, and let $t\in(0,1)$. Then we have
\begin{eqnarray*}
H_c\big( t\eta_1 + (1-t)\eta_2 \big)
	& = &
	L_c\big( t\eta_1 + (1-t)\eta_2, \, t\eta_1 + (1-t)\eta_2 \big)
\\
	& = &
	t \, L_c\big( \eta_1, \, t\eta_1 + (1-t)\eta_2 \big) + (1-t) \, L_c\big( \eta_2, \, t\eta_1 + (1-t)\eta_2 \big) 
\\
	& > &
	t \, H_c(\eta_1) + (1-t) \, H_c(\eta_2)
	\,.
\end{eqnarray*}
Thus $H_c$ is strictly concave.
\end{proof}

%---------- Proper Composite Losses ------------
\vspace{10pt}
\noindent
\textbf{Proper Composite Losses.}
The notion of properness can be extended to binary loss functions operating on prediction spaces $\hat{\Y}$ other than $[0,1]$ via composition with a \emph{link} function $\psi:[0,1]\>\hat{\Y}$. Specifically, for any $\hat{\Y}\subseteq\R^*$, a loss function $\ell:\{\pm1\}\times\hat{\Y}\>\R_+$ is said to be \emph{proper composite} if it can be written as
\begin{equation}
    \ell(y,\hat{y}) = c(y,\psi^{-1}(\hat{y}))
\end{equation}
for some proper loss $c:\{\pm1\}\times[0,1]\>\R_+^*$ and strictly increasing (and therefore invertible) link function $\psi:[0,1]\>\hat{\Y}$.
Proper composite losses have been studied recently in \cite{ReidWi10,ReidWi11,Buja+05}, and include several widely used losses such as squared, squared hinge, logistic, and exponential losses. 

It is worth noting that for a proper composite loss $\ell$ formed from a proper loss $c$, $H_\ell = H_c$. Moreover, any property associated with the underlying proper loss $c$ can also be used to describe the composite loss $\ell$; thus we will refer to a proper composite loss $\ell$ formed from a regular proper loss $c$ as \emph{regular proper composite}, a composite loss formed from a strictly proper loss as \emph{strictly proper composite}, etc. In \Sec{sec:strongly-proper}, we will define and characterize \emph{strongly proper} (composite) losses, which we will use to obtain regret bounds for bipartite ranking.

%========== SECTION 3 ===========
\section{Related Work}
\label{sec:related}

As noted above, a popular theoretical and algorithmic framework for bipartite ranking has been to reduce the problem to pairwise classification. Below we describe this reduction in the context of our setting and notation, and then review the result of \cite{Kotlowski+11} which builds on this pairwise reduction.

%------------ Section 3.1 --------------
\subsection{Reduction of Bipartite Ranking to Pairwise Binary Classification}
\label{subsec:reduction-pairwise}

For any distribution $D$ on $\X\times\{\pm1\}$, consider the distribution $\tilde{D}$ on $(\X\times\X)\times\{\pm1\}$ defined as follows: 
\begin{enumerate}
\item 
Sample $(X,Y)$ and $(X',Y')$ i.i.d.\ from $D$;
\item
If $Y=Y'$, then go to step 1; else set\footnote{Throughout the paper, $\sign(u) = +1$ if $u>0$ and $-1$ otherwise.}
\[
\tilde{X} = (X,X') \,, ~~ \tilde{Y} = \sign(Y-Y')
\]
and return $(\tilde{X},\tilde{Y})$.
\end{enumerate}
Then it is easy to see that, under $\tilde{D}$,
\begin{equation}
\P\big( \tilde{X} = (x,x') \big)
    ~ = ~
    \frac{\P(X=x)\, \P(X'=x')\, \big( \eta(x)(1-\eta(x')) + \eta(x')(1-\eta(x)) \big)}{2p(1-p)}
\end{equation}
\begin{equation}
\tilde{\eta}((x,x'))
    ~ = ~
    \P\big( \tilde{Y} = 1\mid \tilde{X} = (x,x') \big)
    ~ = ~
    \frac{\eta(x)(1-\eta(x'))}{\eta(x)(1-\eta(x')) + \eta(x')(1-\eta(x))}
\end{equation}
\begin{equation}
\tilde{p}
    ~ = ~
\P\big( \tilde{Y} = 1 \big) 
    ~ = ~
    \half
    \,.
\end{equation}
Moreover, for the 0-1 loss $\ell_\zo:\{\pm1\}\times\{\pm1\}\>\{0,1\}$ given by $\ell_\zo(y,\hat{y}) = \1(\hat{y}\neq y)$, we have the following for any pairwise binary classifier $h:\X\times\X\>\{\pm1\}$:
\begin{eqnarray}
\er_{\tilde{D}}^\zo[h] 
    & = &
    \E_{(\tilde{X},\tilde{Y})\sim\tilde{D}} \Big[ \1\big( h(\tilde{X}) \neq \tilde{Y} \big) \Big]
\\
\er_{\tilde{D}}^{\zo,*}
    & = &
    \E_{\tilde{X}}\Big[ \min\Big( \tilde{\eta}(\tilde{X}), 1-\tilde{\eta}(\tilde{X}) \Big) \Big]
\\
\reg_{\tilde{D}}^\zo[h]
    & = &
    \er_{\tilde{D}}^\zo[h] - \er_{\tilde{D}}^{\zo,*}
    \,.
\end{eqnarray}
Now for any scoring function $f:\X\>\R^*$, define $f_\diff:\X\times\X\>\R^*$ as 
\begin{equation}
f_\diff(x,x') 
	~ = ~
	f(x) - f(x')
	\,.
\label{eqn:f-diff}
\end{equation}
Then it is easy to see that:
\begin{eqnarray}
\er_D^\rank[f]
    & = &
    \er_{\tilde{D}}^\zo[\sign \circ f_\diff] 
\\
\er_D^{\rank,*}
    & = &
    \er_{\tilde{D}}^{\zo,*}
    \,,
\label{eqn:rank-error-opt-pairwise}
\end{eqnarray}
where $(g\circ f)(u) = g(f(u))$.
The equality in \Eqn{eqn:rank-error-opt-pairwise} follows from the fact that the classifier $h^*(x,x') = \sign(\eta(x) - \eta(x'))$
achieves the Bayes 0-1 risk, i.e.\  $\er_{\tilde{D}}^\zo[h^*] = \er_{\tilde{D}}^{\zo,*}$ \cite{Clemencon+08}. 
Thus
\begin{eqnarray}
\reg_D^\rank[f]
    & = &
    \reg_{\tilde{D}}^\zo[\sign \circ f_\diff]
    \,,
\label{eqn:regret-bound-pairwise}
\end{eqnarray}
and therefore the ranking regret of a scoring function $f:\X\>\R^*$ can be analyzed via upper bounds on the 0-1 regret of the pairwise classifier $(\sign\circ f_\diff):\X\times\X\>\{\pm1\}$.\footnote{Note that the setting here is somewhat different from that of \cite{Balcan+08} and \cite{AilonMo08}, who consider a \emph{subset} version of bipartite ranking where each instance consists of some finite subset of objects to be ranked; there also the problem is reduced to a (subset) pairwise classification problem, and it is shown that given any (subset) pairwise classifier $h$, a subset ranking function $f$ can be constructed such that the resulting subset ranking regret is at most twice the subset pairwise classification regret of $h$ \cite{Balcan+08}, or in expectation at most equal to the pairwise classification regret of $h$ \cite{AilonMo08}.}

In particular, as noted in \cite{Clemencon+08}, applying a result of \cite{Bartlett+06}, we can upper bound the pairwise 0-1 regret above in terms of the pairwise $\ell_\phi$-regret associated with any classification-calibrated margin loss $\ell_\phi:\{\pm1\}\times\R^*\>\R_+^*$, i.e.\ any loss of the form $\ell_\phi(y,\hat{y}) = \phi(y\hat{y})$ for some function $\phi:\R^*\>\R_+^*$ satisfying $\forall~ \eta\in[0,1], \eta\neq\half$,\footnote{We abbreviate $L_\phi = L_{\ell_\phi}$, $\er_D^\phi = \er_D^{\ell_\phi}$, etc.}
\begin{equation}
\hat{y}^* \in\arg\min_{\hat{y}\in\R^*} L_\phi(\eta,\hat{y}) 
	~ \implies ~
	\hat{y}^*(\eta-\half) > 0
	\,.
\end{equation}
We note in particular that for every proper composite margin loss, the associated link function $\psi$ satisfies $\psi(\half) = 0$ \cite{ReidWi10}, and therefore every strictly proper composite margin loss is classification-calibrated in the sense above.\footnote{We note that in general, every strictly proper (composite) loss is classification-calibrated with respect to any cost-sensitive zero-one loss, using a more general definition of classification calibration with an appropriate threshold (e.g.\ see \cite{ReidWi10}).}  
\begin{theorem}[\cite{Bartlett+06}; see also \cite{Clemencon+08}]
\label{thm:bartlett+06}
Let $\phi:\R^*\>\R_+^*$ be such that the margin loss $\ell_\phi:\{\pm1\}\times\R^*\>\R_+^*$ defined as $\ell_\phi(y,\hat{y}) = \phi(y\hat{y})$ is classification-calibrated as above. Then $\exists$ strictly increasing function $g_\phi:\R_+^*\>[0,1]$ with $g_\phi(0) = 0$ such that for any $\tilde{f}:\X\times\X\>\R^*$,
\[
\reg_{\tilde{D}}^\zo[\sign\circ \tilde{f}]
    ~ \leq ~
    g_\phi \Big( \reg_{\tilde{D}}^\phi[\tilde{f}] \Big)
    \,.
\]
\end{theorem}
\cite{Bartlett+06} give a construction for $g_\phi$; in particular, for the exponential loss given by $\phi_{\exp}(u) = e^{-u}$ and logistic loss given by $\phi_{\log}(u) = \ln(1+e^{-u})$, both of which are strictly proper composite losses  (see \Sec{subsec:examples}) and are therefore classification-calibrated, one has
\begin{eqnarray}
g_{\exp}(z) 
	& \leq & 
	\sqrt{2z} 
\label{eqn:g-exp}
\\
g_{\log}(z) 
	& \leq & 
	\sqrt{2z} 
\,.
\label{eqn:g-log}
\end{eqnarray}
As we describe below, \cite{Kotlowski+11} build on these observations to bound the ranking regret in terms of the regret associated with balanced versions of the exponential and logistic losses.

%------------ Section 3.2 --------------
\subsection{Result of Kotlowski et al.\ (2011)}
\label{subsec:kotlowski+11}

For any binary loss $\ell:\{\pm1\}\times\hat{\Y}\>\R_+^*$, consider defining a \emph{balanced} loss $\ell_\bal:\{\pm1\}\times\hat{\Y}\>\R_+^*$ as
\begin{equation}
\ell_{\bal}(y,\hat{y})
	~ = ~
         \frac{1}{2p} \ell(1,\hat{y}) \cdot \1(y=1) + \frac{1}{2(1-p)} \ell(-1,\hat{y}) \cdot \1(y=-1)
         \,.
\end{equation}
Note that such a balanced loss depends on the underlying distribution $D$ via $p=\P(Y=1)$. Then \cite{Kotlowski+11} show the following, via analyses specific to the exponential and logistic losses:
\begin{theorem}[\cite{Kotlowski+11}]
\label{thm:kotlowski+11}
For any $f:\X\>\R^*$,
\begin{eqnarray*}
\reg_{\tilde{D}}^{\exp}[f_\diff]
    & \leq &
    \frac{9}{4} \reg_D^{\exp,\bal}[f]
\\
\reg_{\tilde{D}}^{\log}[f_\diff]
    & \leq &
    2\, \reg_D^{\log,\bal}[f]
    \,.
\end{eqnarray*}
\end{theorem}
Combining this with the results of \Eqn{eqn:regret-bound-pairwise}, \Thm{thm:bartlett+06}, and \Eqs{eqn:g-exp}{eqn:g-log}
then gives the following bounds on the ranking regret of any scoring function $f:\X\>\R^*$ in terms of the (non-pairwise) balanced exponential and logistic regrets of $f$:
\begin{eqnarray}
\reg_D^\rank[f]
    & \leq &
    \frac{3}{\sqrt{2}} \, \sqrt{\reg_D^{\exp,\bal}[f]}
\\
\reg_D^\rank[f]
    & \leq &
    2 \, \sqrt{\reg_D^{\log,\bal}[f]}
    \,.
\end{eqnarray}
This suggests that an algorithm that produces a function $f:\X\>\R^*$ with low balanced exponential or logistic regret will also have low ranking regret. Unfortunately, since the balanced losses depend on the unknown distribution $D$, they cannot be optimized by an algorithm directly.\footnote{We note it is possible to optimize approximately balanced losses, e.g.\ by estimating $p$ from the data.} \cite{Kotlowski+11} provide some justification for why in certain situations, minimizing the usual exponential or logistic loss may also minimize the balanced versions of these losses; however, by doing so, one loses the quantitative nature of the above bounds. Below we obtain upper bounds on the ranking regret of a function $f$ directly in terms of its loss-based regret (with no balancing terms) for a wide range of proper (composite) loss functions that we term \emph{strongly proper}, including the exponential and logistic losses as special cases.

%========== SECTION 4 ===========
\section{Strongly Proper Losses}
\label{sec:strongly-proper}

We define strongly proper losses as follows:

\begin{definition}
Let $c:\{\pm1\}\times[0,1]\>\R_+^*$ be a binary CPE loss and let $\lambda > 0$. We say $c$ is \emph{$\lambda$-strongly proper} if for all $\eta,\hat{\eta}\in[0,1]$,
\[
L_c(\eta,\hat{\eta}) - H_c(\eta) 
	~ \geq ~
	\frac{\lambda}{2} (\eta - \hat{\eta})^2
	\,.
\]
\end{definition}

We have the following necessary and sufficient conditions for strong properness:

\begin{lemma}
\label{lem:strongly-proper-necessary}
Let $\lambda > 0$. If $c:\{\pm1\}\times[0,1]\>\R_+^*$ is $\lambda$-strongly proper, then $H_c$ is $\lambda$-strongly concave.
\end{lemma}
\begin{proof}
The proof is similar to the `only if' direction in the proof of \Thm{thm:strictly-proper}.
Let $c$ be $\lambda$-strongly proper. Let $\eta_1,\eta_2\in[0,1]$ such that $\eta_1\neq \eta_2$, and let $t\in(0,1)$. Then we have
\begin{eqnarray*}
H_c\big( t\eta_1 + (1-t)\eta_2 \big)
	& = &
	L_c\big( t\eta_1 + (1-t)\eta_2, \, t\eta_1 + (1-t)\eta_2 \big)
\\
	& = &
	t \, L_c\big( \eta_1, \, t\eta_1 + (1-t)\eta_2 \big) + (1-t) \, L_c\big( \eta_2, \, t\eta_1 + (1-t)\eta_2 \big) 
\\
	& \geq &
	t \, \left( H_c(\eta_1) + \frac{\lambda}{2} (1-t)^2 (\eta_1-\eta_2)^2 \right) + 
	(1-t) \, \left( H_c(\eta_2) + \frac{\lambda}{2} t^2 (\eta_1-\eta_2)^2 \right)
\\
	& = &
	t  \, H_c(\eta_1) + (1-t) \, H_c(\eta_2) + \frac{\lambda}{2} t(1-t) (\eta_1-\eta_2)^2 
	\,.
\end{eqnarray*}
Thus $H_c$ is $\lambda$-strongly concave.
\end{proof}

\begin{lemma}
\label{lem:strongly-proper-sufficient}
Let $\lambda > 0$ and let $c:\{\pm1\}\times[0,1]\>\R_+^*$ be a regular proper loss. If $H_c$ is $\lambda$-strongly concave, then $c$ is $\lambda$-strongly proper.
\end{lemma}
\begin{proof}
Let $\eta,\hat{\eta}\in[0,1]$. By \Thm{thm:savage71}, there exists a superderivative $H'_c(\hat{\eta})$ of $H_c$ at $\hat{\eta}$ such that 
\[
L_c(\eta,\hat{\eta}) 
	~ = ~ 
	H_c(\hat{\eta}) + (\eta - \hat{\eta}) \cdot H'_c(\hat{\eta})
	\,.
\]
This gives 
\begin{eqnarray*}
L_c(\eta,\hat{\eta}) - H_c(\eta)
	& = & 
	H_c(\hat{\eta}) - H_c(\eta) + (\eta - \hat{\eta}) \cdot H'_c(\hat{\eta})
\\
	& \geq &
	\frac{\lambda}{2} (\hat{\eta} - \eta)^2
	\,,
	~~~~\mbox{since $H_c$ is $\lambda$-strongly concave.}
\end{eqnarray*}
Thus $c$ is $\lambda$-strongly proper.
\end{proof}

This gives us the following characterization of strong properness for regular proper losses:
\begin{theorem}
\label{thm:strongly-proper}
Let $\lambda > 0$ and let $c:\{\pm1\}\times[0,1]\>\R_+^*$ be a regular proper loss. Then $c$ is $\lambda$-strongly proper if and only if $H_c$ is $\lambda$-strongly concave.
\end{theorem}
Several examples of strongly proper (composite) losses will be provided in \Sec{subsec:examples} and \Sec{subsec:constructing-strongly-proper}. \Thm{thm:strongly-proper} will form our main tool in establishing strong properness of many of these loss functions.

%========== SECTION 5 ===========
\section{Regret Bounds via Strongly Proper Losses}
\label{sec:bounds}

We start by recalling the following result of \cite{Clemencon+08} (adapted to account for ties, and for the conditioning on $Y\neq Y'$):
\begin{theorem}[\cite{Clemencon+08}]
\label{thm:clemencon+08}
For any $f:\X\>\R^*$,
\begin{eqnarray*}
\reg_D^\rank[f]
	& = &
	\frac{1}{2p(1-p)} \E_{X,X'}\Big[ 
		\big| \eta(X) - \eta(X') \big| \cdot \Big( 
			\1\big( (f(X)-f(X'))(\eta(X)-\eta(X')) < 0 \big) 
		\Big.
	\Big.
\\
	& & 
	\hspace{6cm}
	\Big. 
		\Big.
			+~ \half \1\big( f(X)=f(X') \big)
		\Big)
	\Big]
	\,.
\end{eqnarray*}
\end{theorem}
As noted by \cite{ClemenconRo11}, this leads to the following corollary on the regret of any plug-in ranking function based on an estimate $\hat{\eta}$:
\begin{corollary}
\label{cor:regret-bound-plugin}
For any $\hat{\eta}:\X\>[0,1]$,
\[
\reg_D^\rank\big[ \, \hat{\eta} \, \big]
    ~ \leq ~
    \frac{1}{p(1-p)} \E_{X}\big[ \big| \hat{\eta}(X) - \eta(X) \big| \big]
    \,.
\]
\end{corollary}
For completeness, a proof is given in \App{app:A}. 
We are now ready to prove our main result. 

%---------- Section 5.1 ------------
\subsection{Main Result}
\label{subsec:main-result}

\begin{theorem}
\label{thm:regret-bound-proper}
Let $\hat{\Y}\subseteq\R^*$ and let $\lambda>0$. Let $\ell:\{\pm1\}\times\hat{\Y}\>\R_+^*$ be a $\lambda$-strongly proper composite loss. Then for any $f:\X\>\hat{\Y}$,
\[
\reg_D^\rank[f] 
	~ \leq ~
	\frac{\sqrt{2}}{p(1-p)\sqrt{\lambda}} \sqrt{\reg_D^\ell[f]}
	\,.
\]
\end{theorem}
\begin{proof}
Let $c:\{\pm1\}\times[0,1]\>\R_+^*$ be a $\lambda$-strongly proper loss and $\psi:[0,1]\>\hat{\Y}$ be a (strictly increasing) link function such that $\ell(y,\hat{y}) = c(y,\psi^{-1}(\hat{y}))$ for all $y\in\{\pm1\},\hat{y}\in\hat{\Y}$. Let $f:\X\>\hat{\Y}$.
Then we have, 
\begin{eqnarray*}
\reg_D^\rank[f]
	& = &
	\reg_D^\rank[ \psi^{-1}\circ f ]
	\,, ~~~\mbox{since $\psi$ is strictly increasing}
\\
	& \leq &
	\frac{1}{p(1-p)} \E_X\big[ \big| \psi^{-1}(f(X)) - \eta(X) \big| \big]
	\,, ~~~\mbox{by \Cor{cor:regret-bound-plugin}}
\\
	& = &
	\frac{1}{p(1-p)} \sqrt{ \Big( \E_X\big[ \big| \psi^{-1}(f(X)) - \eta(X) \big| \big] \Big)^2 }
\\
	& \leq &
	\frac{1}{p(1-p)} \sqrt{\E_X\Big[ \big( \psi^{-1}(f(X)) - \eta(X) \big)^2 \Big]}
	\,,
\\[2pt]
	& &
	\hspace{3cm} ~~~\mbox{by convexity of $\phi(u)=u^2$ and Jensen's inequality}
\\[2pt]
	& \leq &
	\frac{1}{p(1-p)} \sqrt{ \frac{2}{\lambda}\,\E_X\big[ R_c(\eta(X),\psi^{-1}(f(X))) \big]}
	\,, ~~~\mbox{since $c$ is $\lambda$-strongly proper}
\\[2pt]
	& = &
	\frac{1}{p(1-p)} \sqrt{ \frac{2}{\lambda}\,\E_X\big[ R_\ell(\eta(X),(f(X)) \big]}
\\
	& = &
	\frac{\sqrt{2}}{p(1-p)\sqrt{\lambda}} \sqrt{\reg_D^\ell[f]}	
	\,.
\end{eqnarray*}
\end{proof}

\Thm{thm:regret-bound-proper} shows that for any strongly proper composite loss $\ell:\{\pm1\}\times\hat{\Y}\>\R_+^*$, 
a function $f:\X\>\hat{\Y}$ with low $\ell$-regret will also have low ranking regret. Below we give several examples of such strongly proper (composite) loss functions; properties of some of these losses are summarized in \Tab{tab:loss-functions}. 

%------------- Section 5.2 ---------------
\subsection{Examples}
\label{subsec:examples}

\begin{example}[Exponential loss]
The exponential loss $\ell_{\exp}:\{\pm1\}\times\R^*\>\R_+^*$ defined as 
\[
\ell_{\exp}(y,\hat{y}) 
	~ = ~ 
	e^{-y\hat{y}}
\]
is a proper composite loss with associated proper loss $c_{\exp}:\{\pm1\}\times[0,1]\>\R_+^*$ and link function $\psi_{\exp}:[0,1]\>\R^*$ given by
\[
c_{\exp}(y,\hat{\eta}) 
	~ = ~
	\left( \frac{1-\hat{\eta}}{\hat{\eta}}\right )^{y/2}
	\,; ~~~~
\psi_{\exp}(\hat{\eta})
	~ = ~
	\frac{1}{2} \ln\left( \frac{\hat{\eta}}{1-\hat{\eta}} \right)
	\,.
\]
It is easily verified that $c_{\exp}$ is regular. Moreover, it can be seen that 
\[
H_{\exp}(\eta) 
	~ = ~ 
	2 \sqrt{\eta(1-\eta)}
	\,,
\]
with
\[
- H''_{\exp}(\eta) 
	~ = ~
	\frac{1}{2(\eta(1-\eta))^{3/2}}
	~ \geq ~
	4 ~~~~\forall\eta\in[0,1] 
	\,.
\]
Thus $H_{\exp}$ is 4-strongly concave, and so by \Thm{thm:strongly-proper}, we have $\ell_{\exp}$ is 4-strongly proper composite. Therefore applying \Thm{thm:regret-bound-proper} we have for any $f:\X\>\R^*$,
\[
\reg_D^\rank[f] 
	~ \leq ~
	\frac{1}{\sqrt{2}\,p(1-p)} \sqrt{\reg_D^{\exp}[f]}
	\,.
\]
\end{example}

\begin{example}[Logistic loss]
The logistic loss $\ell_{\exp}:\{\pm1\}\times\R^*\>\R_+^*$ defined as 
\[
\ell_{\log}(y,\hat{y}) 
	~ = ~ 
	\ln(1+e^{-y\hat{y}})
\]
is a proper composite loss with associated proper loss $c_{\log}:\{\pm1\}\times[0,1]\>\R_+^*$ and link function $\psi_{\log}:[0,1]\>\R^*$ given by
\[
c_{\log}(1,\hat{\eta}) 
	~ = ~
	-\ln \hat{\eta}
	\,; ~~~~
c_{\log}(-1,\hat{\eta}) 
	~ = ~
	-\ln(1-\hat{\eta})
	\,; ~~~~
\psi_{\log}(\hat{\eta})
	~ = ~
	\ln\left( \frac{\hat{\eta}}{1-\hat{\eta}} \right)
	\,.
\]
Again, it is easily verified that $c_{\log}$ is regular. Moreover, it can be seen that 
\[
H_{\log}(\eta) 
	~ = ~ 
	-\eta \ln\eta - (1-\eta) \ln(1-\eta)
	\,,
\]
with
\[
- H''_{\log}(\eta) 
	~ = ~
	\frac{1}{\eta(1-\eta)}
	~ \geq ~
	4 ~~~~\forall\eta\in[0,1] 
	\,.
\]
Thus $H_{\log}$ is 4-strongly concave, and so by \Thm{thm:strongly-proper}, we have $\ell_{\log}$ is 4-strongly proper composite. Therefore applying \Thm{thm:regret-bound-proper} we have for any $f:\X\>\R^*$,
\[
\reg_D^\rank[f] 
	~ \leq ~
	\frac{1}{\sqrt{2}\,p(1-p)} \sqrt{\reg_D^{\log}[f]}
	\,.
\]
\end{example}

\begin{example}[Squared and squared hinge losses]
The (binary) squared loss $(1-y\hat{y})^2$ and squared hinge loss $((1-y\hat{y})_+)^2$ (where $u_+ = \max(u,0)$) are generally defined for $\hat{y}\in\R$. To obtain class probability estimates from a predicted value $\hat{y}\in\R$, one then truncates $\hat{y}$ to $[-1,1]$, and uses $\hat{\eta} = \frac{\hat{y}+1}{2}$ \cite{Zhang04}. To obtain a proper loss, we can take $\hat{\Y}=[-1,1]$; in this range, both losses coincide, and we can define $\ell_\sq:\{\pm1\}\times[-1,1]\>[0,4]$ as
\[
\ell_\sq(y,\hat{y}) 
	~ = ~ 
	(1-y\hat{y})^2
	\,.
\]
This is a proper composite loss with associated proper loss $c_\sq:\{\pm1\}\times[-1,1]\>[0,4]$ and link function $\psi_\sq:[0,1]\>[-1,1]$ given by
\[
c_\sq(1,\hat{\eta}) 
	~ = ~
	4 (1-\hat{\eta})^2
	\,; ~~~~
c_\sq(-1,\hat{\eta}) 
	~ = ~
	4 \hat{\eta}^2
	\,; ~~~~
\psi_\sq(\hat{\eta})
	~ = ~
	2\hat{\eta} - 1
	\,.
\]
It can be seen that 
\[
L_{\sq}(\eta,\hat{\eta}) 
	~ = ~ 
	4\eta (1-\hat{\eta})^2 + 4(1-\eta) \hat{\eta}^2
\]
and 
\[
H_\sq(\eta) 
	~ = ~ 
	4\eta(1-\eta)
	\,,
\]
so that
\[
L_{\sq}(\eta,\hat{\eta}) - H_\sq(\eta)
	~ = ~
	4(\eta - \hat{\eta})^2
	\,.
\]
Thus $\ell_\sq$ is 8-strongly proper composite, and so applying \Thm{thm:regret-bound-proper} we have for any $f:\X\>[-1,1]$,
\[
\reg_D^\rank[f] 
	~ \leq ~
	\frac{1}{2\,p(1-p)} \sqrt{\reg_D^\sq[f]}
	\,.
\]
Note that, if a function $f:\X\>\R$ is learned, then our bound in terms of $\ell_\sq$-regret applies to the ranking regret of an appropriately transformed function $\bar{f}:\X\>[-1,1]$, such as that obtained by truncating values $f(x)\notin[-1,1]$ to the appropriate endpoint $-1$ or $1$:
\[
\bar{f}(x) ~ = ~
	\left\{ \begin{array}{cl}
		-1 & ~~\mbox{if $f(x)<-1$} \\
		f(x) & ~~\mbox{if $f(x)\in[-1,1]$} \\
		1 & ~~\mbox{if $f(x)>1$.} 
	\end{array} \right.
\]  
\end{example}

%------------- Table: Loss Functions --------------

\begin{table}[t]
\caption{Examples of strongly proper composite losses $\ell:\{\pm1\}\times\hat{\Y}\>\R_+^*$ satisfying the conditions of \Thm{thm:regret-bound-proper}, together with prediction space $\hat{\Y}$, proper loss $c:\{\pm1\}\times[0,1]\>\R_+^*$, link function $\psi:[0,1]\>\hat{\Y}$, and strong properness parameter $\lambda$.}
\label{tab:loss-functions}
\begin{center}
\vspace{-12pt}
\begin{small}
\begin{tabular}{|@{~}c@{~}||@{~}c@{~}|@{~}c@{~}|@{~}c@{~}|@{~}c@{~}|@{~}c@{~}|c@{~}|}
\hline
%\hline
\textbf{Loss} & $\hat{\Y}$ &
    $\ell(y,\hat{y})$ & 
    \multicolumn{2}{c|@{~}}{$c(y,\hat{\eta})$~~~~} & 
    $\psi(\hat{\eta})$ & 
    $\lambda$ \rule{0pt}{13pt} \\[-7pt]
& & & \multicolumn{2}{@{}c|@{~}}{\rule{5.2cm}{0.5pt}} & & \\[-3pt]
& & & $y=1$ & $y=-1$ & & \\[2pt]
\hline
\hline
Exponential & $\R^*$ &
    $e^{-y\hat{y}} $ & 
    $\sqrt{\frac{1-\hat{\eta}}{\hat{\eta}}}$ & 
    $\sqrt{\frac{\hat{\eta}}{1-\hat{\eta}}}$ &
    $\frac{1}{2} \ln\big( \frac{\hat{\eta}}{1-\hat{\eta}} \big)$ & 
    $4$ \rule{0pt}{15pt} \\[6pt]
\hline
Logistic & $\R^*$ & 
    $\ln(1+e^{-y\hat{y}})$ & 
    $-\ln \hat{\eta}$ & 
    $-\ln(1-\hat{\eta})$ &
    $\ln\big( \frac{\hat{\eta}}{1-\hat{\eta}} \big)$ & 
    $4$ \rule{0pt}{13pt} \\[2pt]
\hline
Squared & $[-1,1]$ &
    $(1-y\hat{y})^2$ & 
    $4(1-\hat{\eta})^2$ & 
    $4{\hat{\eta}}^2$ & 
    $2\hat{\eta}-1$ & 
    $8$ \rule{0pt}{13pt} \\[2pt]
\hline
Spherical & $[0,1]$ &
   $c(y,\hat{y})$ & 
    $1 - \frac{\hat{\eta}}{\sqrt{\hat{\eta}^2 + (1-\hat{\eta})^2}}$ &
    $1 - \frac{1-\hat{\eta}}{\sqrt{\hat{\eta}^2 + (1-\hat{\eta})^2}}$ &
    $\hat{\eta}$ & 
    $1$ \rule{0pt}{15pt} \\[6pt]
\hline
\begin{minipage}{0.9in}\begin{center} Canonical `exponential' \end{center}\end{minipage} & $\R^*$ &
    $\sqrt{1+\big( \frac{\hat{y}}{2} \big)^2} - \frac{y\hat{y}}{2}$ & 
    $\sqrt{\frac{1-\hat{\eta}}{\hat{\eta}}}$ & 
    $\sqrt{\frac{\hat{\eta}}{1-\hat{\eta}}}$ &
    $\frac{2\hat{\eta}-1}{\sqrt{\hat{\eta}(1-\hat{\eta})}}$ &
    $4$ \rule{0pt}{18pt} \\[6pt]
\hline
\begin{minipage}{0.9in}\begin{center} Canonical squared \end{center}\end{minipage} & $[-1,1]$ &
    $\frac{1}{4} (1-y\hat{y})^2$ & 
    $(1-\hat{\eta})^2$ & 
    ${\hat{\eta}}^2$ & 
    $2\hat{\eta}-1$ & 
    $2$ \rule{0pt}{18pt} \\[6pt]
\hline
\begin{minipage}{0.9in}\begin{center} Canonical spherical \end{center}\end{minipage} & $[-1,1]$ &
    $1 - \frac{1}{2}\big( \sqrt{2-\hat{y}^2} + y\hat{y} \big)$ & 
    $1 - \frac{\hat{\eta}}{\sqrt{\hat{\eta}^2 + (1-\hat{\eta})^2}}$ &
    $1 - \frac{1-\hat{\eta}}{\sqrt{\hat{\eta}^2 + (1-\hat{\eta})^2}}$ &
    $\frac{2\hat{\eta}-1}{\sqrt{\hat{\eta}^2 + (1-\hat{\eta})^2}}$ & 
    $1$ \rule{0pt}{18pt} \\[8pt]
\hline
%\hline
\end{tabular}
\end{small}
\end{center}
\end{table}

%------------- Section 5.3 ---------------
\subsection{Constructing Strongly Proper Losses}
\label{subsec:constructing-strongly-proper}

In general, given any concave function $H:[0,1]\>\R_+$, one can construct a proper loss $c:\{\pm1\}\times[0,1]\>\R_+^*$ with $H_c = H$ as follows:
\begin{eqnarray}
c(1,\hat{\eta}) 
	& = &
	H(\hat{\eta}) + (1-\hat{\eta}) H'(\hat{\eta})
\label{eqn:H-c1}
\\
c(-1,\hat{\eta}) 
	& = &
	H(\hat{\eta}) -\hat{\eta} H'(\hat{\eta})
	\,,
\label{eqn:H-c2}
\end{eqnarray}
where $H'(\hat{\eta})$ denotes any superderivative of $H$ at $\hat{\eta}$. 
It can be verified that this gives $L_c(\eta,\hat{\eta}) = H(\hat{\eta}) + (\eta-\hat{\eta}) H'(\hat{\eta})$ for all $\eta,\hat{\eta}\in[0,1]$, and therefore $H_c(\eta) = H(\eta)$ for all $\eta\in[0,1]$. Moreover, if $H$ is such that $H(\hat{\eta}) + (1-\hat{\eta}) H'(\hat{\eta}) \in\R_+ ~\forall\hat{\eta}\in(0,1]$ and $H(\hat{\eta}) -\hat{\eta} H'(\hat{\eta}) \in\R_+ ~\forall\hat{\eta}\in[0,1)$,
then the loss $c$ constructed above is also regular.
Thus, starting with any $\lambda$-strongly concave function $H:[0,1]\>\R_+$ satisfying these regularity conditions, any proper composite loss $\ell$ formed from the loss function $c$ constructed according to \Eqs{eqn:H-c1}{eqn:H-c2} (and any link function $\psi$) is $\lambda$-strongly proper composite.

\begin{example}[Spherical loss]
Consider starting with the function $H_\spher:[0,1]\>\R$ defined as
\[
H_\spher(\eta) 
	~ = ~
	1 - \sqrt{\eta^2 + (1-\eta)^2}
	\,.
\]
Then 
\[
H'_\spher(\eta) 
	~ = ~
	\frac{-(2\eta-1)}{\sqrt{\eta^2 + (1-\eta)^2}}
\]
and
\[
-H''_\spher(\eta)
	~ = ~
	\frac{1}{(\eta^2 + (1-\eta)^2)^{3/2}}
	~ \geq ~
	1 ~~~~\forall \eta\in[0,1] 
	\,,
\]
and therefore $H_\spher$ is 1-strongly concave. Moreover, since $H_\spher$ and $H'_\spher$ are both bounded, the conditions for regularity are also satisfied. Thus we can use \Eqs{eqn:H-c1}{eqn:H-c2} to construct a 1-strongly proper loss $c_\spher:\{\pm1\}\times[0,1]\>\R$ as follows:
\[
c_\spher(1,\hat{\eta}) 
	~ = ~
	H_\spher(\hat{\eta}) + (1-\hat{\eta}) H'_\spher(\hat{\eta})
	~ = ~
	1 - \frac{\hat{\eta}}{\sqrt{\hat{\eta}^2 + (1-\hat{\eta})^2}}
\]
\[
c_\spher(-1,\hat{\eta}) 
	~ = ~
	H_\spher(\hat{\eta}) - \hat{\eta} H'_\spher(\hat{\eta})
	~ = ~
	1 - \frac{1-\hat{\eta}}{\sqrt{\hat{\eta}^2 + (1-\hat{\eta})^2}}
	\,.
\]
Therefore by \Thm{thm:regret-bound-proper}, we have for any $f:\X\>[0,1]$,
\[
\reg_D^\rank[f] 
	~ \leq ~
	\frac{\sqrt{2}}{p(1-p)} \sqrt{\reg_D^{\spher}[f]}
	\,.
\]
The loss $c_\spher$ above corresponds to the \emph{spherical scoring rule} described in \cite{GneitingRa07}.
\end{example}

We also note that, for every strictly proper loss $c:\{\pm1\}\times[0,1]\>\R_+^*$, there is an associated  `canonical' link function $\psi:[0,1]\>\hat{\Y}$ defined as 
\begin{equation}
\psi(\hat{\eta}) 
	~ = ~
	c(-1,\hat{\eta}) - c(1,\hat{\eta})
	\,,
\end{equation}
where $\hat{\Y} = \{\psi(\hat{\eta}): \hat{\eta} \in[0,1]\}$. We refer to composite losses comprised of such a strictly proper loss $c$ with the corresponding canonical link $\psi$ as \emph{canonical} proper composite losses. Clearly, multiplying $c$ by a factor $\alpha>0$ results in the corresponding canonical link $\psi$ also being multiplied by $\alpha$; adding a constant (or a function $\theta(y,\hat{\eta}) = \theta(\hat{\eta})$ that is independent of $y$) to $c$ has no effect on $\psi$. Conversely, given any $\hat{\Y}\subseteq\R^*$ and any (strictly increasing) link function $\psi:[0,1]\>\hat{\Y}$, there is a unique strictly proper loss $c:\{\pm1\}\times[0,1]\>\R_+^*$ (up to addition of constants or functions of the form $\theta(y,\hat{\eta}) = \theta(\hat{\eta})$) for which $\psi$ is canonical; this is obtained using \Eqs{eqn:H-c1}{eqn:H-c2} with $H$ satisfying $H'(\hat{\eta}) = -\psi(\hat{\eta})$ (with possible addition of a term $\theta(\hat{\eta})$ to both $c(1,\hat{\eta})$ and $c(-1,\hat{\eta})$ thus constructed).
Canonical proper composite losses $\ell(y,\hat{y})$ have some desirable properties, including for example convexity in their second argument $\hat{y}$ for each $y\in\{\pm1\}$; we refer the reader to \cite{Buja+05,ReidWi10} for further discussion of such properties.

We note that the logistic loss in Example 2 is a canonical proper composite loss. 
On the other hand, as noted in \cite{Buja+05}, the link $\psi_{\exp}$ associated with the exponential loss in Example 1 is not the canonical link for the proper loss $c_{\exp}$ (see Example 5).
The squared loss in Example 3 is almost canonical, modulo a scaling factor; one needs to scale either the link function or the loss appropriately (Example 6). 

\begin{example}[Canonical proper composite loss associated with $c_{\exp}$] 
Let $c_{\exp}:\{\pm1\}\times[0,1]\>\R_+^*$ be as in Example 1. The corresponding canonical link $\psi_{{\exp},\can}:[0,1]\>\R^*$ is given by 
\[
\psi_{{\exp},\can}(\hat{\eta}) 
	~ = ~
	\sqrt{\frac{\hat{\eta}}{1-\hat{\eta}}} - \sqrt{\frac{1-\hat{\eta}}{\hat{\eta}}}
	~ = ~
	\frac{2\hat{\eta}-1}{\sqrt{\hat{\eta}(1-\hat{\eta})}}
	\,.
\]
With a little algebra, it can be seen that the resulting canonical proper composite loss $\ell_{{\exp},\can}:\{\pm1\}\times\R^*\>\R_+^*$ is given by
\[
\ell_{{\exp},\can}(y,\hat{y})
	~ = ~
	\sqrt{1+\Big( \frac{\hat{y}}{2} \Big)^2} - \frac{y\hat{y}}{2}
	\,.
\]
Since we saw $c_{\exp}$ is 4-strongly proper, we have $\ell_{{\exp},\can}$ is 4-strongly proper composite, and therefore we have from \Thm{thm:regret-bound-proper} that for any $f:\X\>\R^*$,
\[
\reg_D^\rank[f] 
	~ \leq ~
	\frac{1}{\sqrt{2}\,p(1-p)} \sqrt{\reg_D^{{\exp},\can}[f]}
	\,.
\]
\end{example}

\begin{example}[Canonical squared loss]
For $c_\sq:\{\pm1\}\times[0,1]\>[0,4]$ defined as in Example 3, the canonical link $\psi_{\sq,\can}:[0,1]\>\hat{\Y}$ is given by 
\[
\psi_{\sq,\can}(\hat{\eta})
	~ = ~ 
	4  \hat{\eta}^2 - 4 (1-\hat{\eta})^2
	~ = ~
	4(2\hat{\eta}-1)
	\,,
\]
with $\hat{\Y} = [-4,4]$, and the resulting canonical squared loss $\ell_{\sq,\can}:\{\pm1\}\times[-4,4]\>[0,4]$ is given by
\[
\ell_{\sq,\can}(y,\hat{y}) 
	~ = ~ 
	\Big( 1-\frac{y\hat{y}}{4} \Big)^2
	\,.
\]
Since we saw $c_\sq$ is 4-strongly proper, we have $\ell_{\sq,\can}$ is 4-strongly proper composite, giving for any $f:\X\>[-4,4]$,
\[
\reg_D^\rank[f] 
	~ \leq ~
	\frac{1}{2\,p(1-p)} \sqrt{\reg_D^{\sq,\can}[f]}
	\,.
\] 
For practical purposes, this is equivalent to using the loss $\ell_\sq:\{\pm1\}\times[-1,1]\>[0,4]$ defined in Example 3.
Alternatively, we can start with a scaled version of the squared proper loss $c_{\sq'}:\{\pm1\}\times[0,1]\>[0,1]$ defined as 
\[
c_{\sq'}(1,\hat{\eta}) 
	~ = ~
	(1-\hat{\eta})^2
	\,; ~~~~
c_{\sq'}(-1,\hat{\eta}) 
	~ = ~
	\hat{\eta}^2
	\,,
\]
for which the associated canonical link $\psi_{\sq',\can}:[0,1]\>\hat{\Y}$ is given by 
\[
\psi_{\sq',\can}(\hat{\eta})
	~ = ~ 
	\hat{\eta}^2 - (1-\hat{\eta})^2
	~ = ~
	2\hat{\eta}-1
	\,,
\]
with $\hat{\Y} = [-1,1]$, and the resulting canonical squared loss $\ell_{\sq',\can}:\{\pm1\}\times[-1,1]\>[0,1]$ is given by
\[
\ell_{\sq',\can}(y,\hat{y}) 
	~ = ~ 
	\frac{(1-y\hat{y})^2}{4}
	\,.
\]
Again, it can be verified that $c_{\sq'}$ is regular; in this case $H_{\sq'}(\eta) = \eta(1-\eta)$ which is 2-strongly concave, giving that $H_{\sq',\can}$ is 4-strongly proper composite. Therefore applying \Thm{thm:regret-bound-proper} we have for any $f:\X\>[-1,1]$,
\[
\reg_D^\rank[f] 
	~ \leq ~
	\frac{1}{p(1-p)} \sqrt{\reg_D^{\sq',\can}[f]}
	\,.
\]
Again, for practical purposes, this is equivalent to using the loss $\ell_\sq$ defined in Example 3.
\end{example}

\begin{example}[Canonical spherical loss]
For $c_\spher:\{\pm1\}\times[0,1]\>\R$ defined as in Example 4, the canonical link $\psi_{\spher,\can}:[0,1]\>\hat{\Y}$ is given by 
\[
\psi_{\spher,\can}(\hat{\eta})
	~ = ~ 
	\frac{2\hat{\eta}-1}{\sqrt{\hat{\eta}^2 + (1-\hat{\eta})^2}}
	\,,
\]
with $\hat{\Y} = [-1,1]$. The resulting canonical spherical loss $\ell_{\spher,\can}:\{\pm1\}\times[-1,1]\>\R$ is given by
\[
\ell_{\spher,\can}(y,\hat{y}) 
	~ = ~ 
	1 - \frac{1}{2}\Big( \sqrt{2-\hat{y}^2} + y\hat{y} \Big)
	\,.
\]
Since we saw $c_\spher$ is 1-strongly proper, we have $\ell_{\spher,\can}$ is 1-strongly proper composite, and therefore we have from \Thm{thm:regret-bound-proper} that for any $f:\X\>[-1,1]$,
\[
\reg_D^\rank[f] 
	~ \leq ~
	\frac{\sqrt{2}}{p(1-p)} \sqrt{\reg_D^{\spher,\can}[f]}
	\,.
\] 
\end{example}

%========== SECTION 6 ===========
\section{Tighter Bounds under Low-Noise Conditions}
\label{sec:tighter-bounds}

In essence, our results exploit the fact that for a $\lambda$-strongly proper composite loss $\ell$ formed from a $\lambda$-strongly proper loss $c$ and link function $\psi$, given any scoring function $f$, the $L_2(\mu)$ distance (where $\mu$ denotes the marginal density of $D$ on $\X$) between $\psi^{-1}(f(X))$ and $\eta(X)$ (and therefore the $L_1(\mu)$ distance between $\psi^{-1}(f(X))$ and $\eta(X)$, which gives an upper bound on the ranking risk of $f$) can be upper bounded precisely in terms of the $\ell$-regret of $f$. From this perspective, $\hat{\eta} = \psi^{-1}\circ f$ can be treated as a `plug-in' scoring function, which we analyzed via \Cor{cor:regret-bound-plugin}.

Recently, \cite{ClemenconRo11} showed that, under certain low-noise assumptions, one can obtain tighter bounds on the ranking risk of a plug-in scoring function $\hat{\eta}:\X\>[0,1]$ than that offered by \Cor{cor:regret-bound-plugin}. Specifically, \cite{ClemenconRo11} consider the following noise assumption for bipartite ranking (inspired by the noise condition studied in \cite{Tsybakov04} for binary classification):

\vspace{10pt}
\noindent
\textbf{Noise Assumption $\NA(\alpha)$ ($\alpha \in [0,1]$):} \emph{A distribution $D$ on $\X\times\{\pm1\}$ satisfies assumption $\NA(\alpha)$ if $\exists$ a constant $C > 0$ such that for all $x\in\X$ and $t\in[0,1]$,}
\[
\P_X\big( \big| \eta(X) - \eta(x) \big| \leq t \big)
	~ \leq ~
	C\cdot t^\alpha
	\,.
\]
Note that $\alpha=0$ imposes no restriction on $D$, while larger values of $\alpha$ impose greater restrictions. \cite{ClemenconRo11} showed the following result (adapted slightly to our setting, where the ranking risk is conditioned on $Y\neq Y'$):

\begin{theorem}[\cite{ClemenconRo11}]
\label{thm:regret-bound-plugin-low-noise}
Let $\alpha \in [0,1)$ and $q\in[1,\infty)$. Then $\exists$ a constant $C_{\alpha,q} > 0$ such that for any distribution $D$ on $\X\times\{\pm1\}$ satisfying noise assumption $\NA(\alpha)$ and any $\hat{\eta}:\X\>[0,1]$,
\[
\reg_D^\rank[ \, \hat{\eta} \, ]
	~ \leq ~
	\frac{C_{\alpha,q}}{p(1-p)} \Big( \E_X\big[ \big| \hat{\eta}(X) - \eta(X) \big|^q \big] \Big)^{\frac{1+\alpha}{q+\alpha}}
	\,.
\]
\end{theorem}

This allows us to obtain the following tighter version of our regret bound in terms of strongly proper losses under the same noise assumption:

\begin{theorem}
\label{thm:regret-bound-proper-low-noise}
Let $\hat{\Y}\subseteq\R^*$ and $\lambda>0$, and let $\alpha \in [0,1)$. Let $\ell:\{\pm1\}\times\hat{\Y}\>\R_+^*$ be a $\lambda$-strongly proper composite loss. Then $\exists$ a constant $C_\alpha > 0$ such that for any distribution $D$ on $\X\times\{\pm1\}$ satisfying noise assumption $\NA(\alpha)$ and any $f:\X\>\hat{\Y}$,
\[
\reg_D^\rank[f] 
	~ \leq ~
	\frac{C_\alpha}{p(1-p)} \left( \frac{2}{\lambda} \right)^{\frac{1+\alpha}{2+\alpha}}
		\Big( \reg_D^\ell[f] \Big)^{\frac{1+\alpha}{2+\alpha}}
	\,.
\]
\end{theorem}
\begin{proof}
Let $c:\{\pm1\}\times[0,1]\>\R_+^*$ be a $\lambda$-strongly proper loss and $\psi:[0,1]\>\hat{\Y}$ be a (strictly increasing) link function such that $\ell(y,\hat{y}) = c(y,\psi^{-1}(\hat{y}))$ for all $y\in\{\pm1\},\hat{y}\in\hat{\Y}$.
Let $D$ be a distribution on $\X\times\{\pm1\}$ satisfying noise assumption $\NA(\alpha)$ and let $f:\X\>\hat{\Y}$.
Then we have, 
\begin{eqnarray*}
\reg_D^\rank[f]
	& = &
	\reg_D^\rank[ \psi^{-1}\circ f ]
	\,, ~~~\mbox{since $\psi$ is strictly increasing}
\\
	& \leq &
	\frac{C_{\alpha,2}}{p(1-p)} \Big( \E_X\Big[ \big( \psi^{-1}(f(X)) - \eta(X) \big)^2 \Big] \Big)^{\frac{1+\alpha}{2+\alpha}}
	\,, 
\\[2pt]
	& &
	\hspace{3cm} ~~~\mbox{by \Thm{thm:regret-bound-plugin-low-noise}, taking $q=2$}
\\[2pt]
	& \leq &
	\frac{C_{\alpha,2}}{p(1-p)} \left( \frac{2}{\lambda}\,\E_X\big[ R_c(\eta(X),\psi^{-1}(f(X))) \big] \right)^{\frac{1+\alpha}{2+\alpha}}
	\,,
\\[2pt]
	& & 
	\hspace{3cm} ~~~\mbox{since $c$ is $\lambda$-strongly proper}
\\[2pt]
	& = &
	\frac{C_{\alpha,2}}{p(1-p)} \left( \frac{2}{\lambda}\,\E_X\big[ R_\ell(\eta(X),f(X)) \big] \right)^{\frac{1+\alpha}{2+\alpha}}
\\
	& = &
	\frac{C_{\alpha,2}}{p(1-p)} \left( \frac{2}{\lambda} \right)^{\frac{1+\alpha}{2+\alpha}}
		\Big( \reg_D^\ell[f] \Big)^{\frac{1+\alpha}{2+\alpha}}
	\,.
\end{eqnarray*}
The result follows by setting $C_\alpha = C_{\alpha,2}$.
\end{proof}

For $\alpha = 0$, as noted above, there is no restriction on $D$, and so the above result gives the same dependence on $\reg_D^\ell[f]$ as that obtained from \Thm{thm:regret-bound-proper}. On the other hand, 
as $\alpha$ approaches 1, the exponent of the $\reg_D^\ell[f]$ term in the above bound approaches $\frac{2}{3}$, which improves over the exponent of $\half$ in \Thm{thm:regret-bound-proper}.

%========== SECTION 7 ===========
\section{Conclusion and Open Questions}
\label{sec:concl}

We have obtained upper bounds on the bipartite ranking regret of a scoring function in terms of the (non-pairwise) regret associated with a broad class of proper (composite) losses that we have termed \emph{strongly proper} (composite) losses. This class includes several widely used losses such as exponential, logistic, squared and squared hinge losses as special cases. 

The definition and characterization of strongly proper losses may be of interest in its own right, and may find applications elsewhere. An open question concerns the necessity of the regularity condition in the characterization of strong properness of a proper loss in terms of strong concavity of the conditional Bayes risk (\Thm{thm:strongly-proper}). The characterization of strict properness of a proper loss in terms of strict concavity of the conditional Bayes risk (\Thm{thm:strictly-proper}) does not require such an assumption, and one wonders whether it may be possible to remove the regularity assumption in the case of strong properness as well.

Many of the strongly proper composite losses that we have considered, such as the exponential, logistic, squared and spherical losses, are margin-based losses, which means the bipartite ranking regret can also be upper bounded in terms of the regret associated with pairwise versions of these losses via the reduction to pairwise classification (\Sec{subsec:reduction-pairwise}). 
A natural question that arises is whether it is possible to characterize conditions on the distribution under which algorithms based on one of the two approaches (minimizing a pairwise form of the loss as in RankBoost/pairwise logistic regression, or minimizing the standard loss as in AdaBoost/standard logistic regression) lead to faster convergence than those based on the other. 
We hope the tools and results established here may help in studying such questions in the future.

%======= ACKNOWLEDGMENTS ========

\section*{Acknowledgments}

Thanks to Harish Guruprasad and Arun Rajkumar for helpful discussions. Thanks also to Yoonkyung Lee for inviting me to give a talk at the ASA Conference on Statistical Learning and Data Mining held in Ann Arbor, Michigan, in June 2012; part of this work was done while preparing for that talk. This research was supported in part by a Ramanujan Fellowship from the Department of Science and Technology, Government of India.

%========== APPENDIX ============

\begin{appendix}

%------------- Appendix A ---------------
\section{Proof of \Cor{cor:regret-bound-plugin}}
\label{app:A}

\begin{proof}
Let $\hat{\eta}:\X\>[0,1]$.
By \Thm{thm:clemencon+08}, we have
\begin{eqnarray*}
\reg_D^\rank[ \, \hat{\eta} \, ]
	& \leq &
	\frac{1}{2p(1-p)} \E_{X,X'}\Big[ 
		\big| \eta(X) - \eta(X') \big| \cdot \1\big( (\hat{\eta}(X)-\hat{\eta}(X'))(\eta(X)-\eta(X')) \leq 0 \big) 
	\Big]
	\,.
\end{eqnarray*}
The result follows by observing that for any $x,x'\in\X$, 
\[
(\hat{\eta}(x)-\hat{\eta}(x'))(\eta(x)-\eta(x')) \leq 0
	~ \implies ~
	|\eta(x) - \eta(x')| \leq |\hat{\eta}(x) - \eta(x)| + |\hat{\eta}(x') - \eta(x')|
	\,.
\]
To see this, not that the statement is trivially true if $\eta(x) = \eta(x')$.
If $\eta(x) > \eta(x')$, then we have 
\begin{eqnarray*}
(\hat{\eta}(x)-\hat{\eta}(x'))(\eta(x)-\eta(x')) \leq 0
	& \implies &
	\hat{\eta}(x) \leq \hat{\eta}(x') 
\\
	& \implies & 
	\eta(x) - \eta(x') \leq (\eta(x) - \hat{\eta}(x)) + (\hat{\eta}(x') - \eta(x')) 
\\
	& \implies & 
	\eta(x) - \eta(x') \leq |\eta(x) - \hat{\eta}(x)| + |\hat{\eta}(x') - \eta(x')| 
\\
	 & \implies &
	|\eta(x) - \eta(x')| \leq |\hat{\eta}(x) - \eta(x)| + |\hat{\eta}(x') - \eta(x')|
	\,.
\end{eqnarray*}
The case $\eta(x) < \eta(x')$ can be proved similarly. Thus we have
\begin{eqnarray*}
\reg_D^\rank[ \, \hat{\eta} \, ]
	& \leq &
	\frac{1}{2p(1-p)} \E_{X,X'}\Big[ 
		\big| \hat{\eta}(X) - \eta(X) \big| + \big| \hat{\eta}(X') - \eta(X') \big| 
	\Big]
\\
	& = &
	\frac{1}{p(1-p)} \E_{X}\Big[ 
		\big| \hat{\eta}(X) - \eta(X) \big|
	\Big]
	\,.
\end{eqnarray*}

\end{proof}

\end{appendix}

%======== BIBLIOGRAPHY ==========

%\bibliographystyle{abbrv}
%\bibliography{arxiv-12-regret-auc}

\end{document}